\documentclass[aps,nofootinbib,preprintnumbers,citeautoscript]{revtex4}
\usepackage{amsmath,amssymb}
\usepackage{enumerate}
\usepackage{graphicx}
\usepackage{amsthm}
 \usepackage{mathtools}
 \usepackage{textcomp}
  \usepackage{amssymb}
  \usepackage{amsthm}
  \usepackage{wasysym}
  \usepackage{amsmath}
 \usepackage{lineno,hyperref}
  \usepackage[export]{adjustbox}
  \usepackage{inputenx}
\usepackage{algorithm}
\usepackage{algpseudocode}
\usepackage{multirow}
 \usepackage{caption}
\usepackage{standalone}
  \usepackage{url}

\makeatletter
\def\ps@pprintTitle{%
   \let\@oddhead\@empty
   \let\@evenhead\@empty
   \let\@oddfoot\@empty
   \let\@evenfoot\@oddfoot
}
\makeatother

\begin{document}
\title{QPSO-CD: Quantum-behaved Particle Swarm Optimization Algorithm with Cauchy Distribution}

\author{Amandeep Singh Bhatia$^{1,3,^*}$, Mandeep Kaur Saggi$^{2}$, Shenggen Zheng$^{3}$, Soumya Ranjan Nayak$^{4}$\\
$^{1}$\textit{Chitkara University Institute of Engineering \& Technology, Chitkara University, Punjab, India}\\
$^{2}$\textit{Department of Computer Science, Thapar Institute of Engineering \& Technology, India} \\
$^{3}$\textit{Center for Quantum Computing, Peng Cheng Laboratory, Shenzhen, China} \\
$^{4}$\textit{Amity School Of Engineering \& Technology, Amity University Uttar Pradesh, Noida, India} \\
E-mail: amandeepbhatia.singh@gmail.com$^{1, ^*}$}

\begin{abstract}
Motivated by the particle swarm optimization (PSO) and quantum computing theory, we have presented a quantum variant of PSO (QPSO) mutated with Cauchy operator and natural selection mechanism (QPSO-CD) from evolutionary computations. The performance of proposed hybrid quantum-behaved particle swarm optimization with Cauchy distribution (QPSO-CD) is investigated and compared with its counterparts based on a set of benchmark problems. Moreover, QPSO-CD is employed in well-studied constrained engineering problems to investigate its applicability. Further, the correctness and time complexity of QPSO-CD are analysed and compared with the classical PSO. 
It has been proved that QPSO-CD handles such real-life problems efficiently and can attain superior solutions in most of the problems.  The experimental results shown that QPSO associated with Cauchy distribution and natural selection strategy outperforms other variants in context of stability and convergence. \\

\textbf{Keywords}: Quantum-behaved algorithm, Particle swarm optimization, Engineering design problems, Cauchy distribution, Quantum computing.
\end{abstract}

\maketitle



\newcommand{\myqed}{\rule{2pt}{1em}}
\newenvironment{myproof}{\begin{proof}}{\let\qedsymbol\myqed\end{proof}}
\theoremstyle{plain}
\newtheorem{theorem}{Theorem}
\theoremstyle{definition}
\newtheorem{defn}{Definition}
\newtheorem{exmp}{Example}[section]

\section{Introduction}
In the late 19th century, the theory of classical mechanics experienced several issues in reporting the physical phenomena of light masses and high velocity microscopic particles. In 1920s, Bohr’s atomic theory \cite{bohr1928quantum}, Heisenberg’s discovery of quantum mechanics \cite{robertson1929uncertainty} and Schrödinger’s \cite{wessels1979schrodinger} discovery of wave mechanics influence the conception of a new field i.e. the quantum mechanics. In 1982, Feynman \cite{feynman1982simulating} stated that quantum mechanical systems can be simulated by quantum computers in exponential time, i.e. better than with classical
computers. Till then, the concept of quantum computing was thought to be only a theoretical possibility, but over the last three decades the research has evolved such as to make quantum computing applications a realistic possibility \cite{wang2012handbook}.

In the last two decades, the field of swarm intelligence has got overwhelming response
among research communities. It is inspired by nature and aims to build decentralized and self-organized systems by collective behavior of individual agents with each other and with their environment. The research foundation of swarm intelligence is constructed mostly upon two families of optimization algorithms i.e. ant colony optimization   (Dorigo et at. \cite{dorigo1999ant} and Colorni et al. \cite{colorni1992distributed} 1992); and particle swarm optimization (PSO) (Kennedy \& Eberhart \cite{kennedy1995particle} 1995). Originally, the swarm intelligence is inspired by certain natural behaviors of flocks of birds and swarms of ants.

In the mid 1990s, particle swarm optimization technique  was introduced for continuous optimization, motivated by flocking of birds. The evolution of PSO based bio-inspired techniques has been in an expedite development in the last two decades. It has got
attention from different fields such as inventory planning \cite{wang2014modified}, power systems \cite{alrashidi2009survey}, manufacturing       \cite{yildiz2009novel}, communication networks \cite{latiff2007energy}, support vector machines \cite{lin2008particle}, to estimate binary inspiral signal \cite{wang2010particle}, gravitational waves \cite{normandin2018particle} and many more. Similar to evolutionary genetic algorithm, it is inspired by simulation of social behavior, where each individual is called particle, and  group of individuals is called swarm. In multi-dimensional search space, the position and velocity of each particle represent a probable solution. Particles fly around in a search space seeking potential solution. At each iteration, each particle adjusts its position according to the goal of its own and its neighbors. Each particle in a neighborhood share the information with others \cite{sun2004particle}. Later, each particle keeps the record of best solution experienced so far to update their positions and adjust their velocities accordingly.

	\begin{figure*}[!ht]
	\centering
	\includegraphics[scale=1.25]{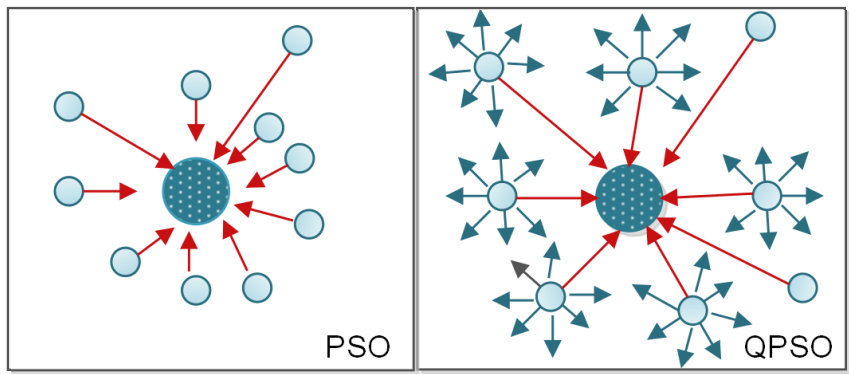}
	\caption{Particles movement in PSO and QPSO algorithm}
\end{figure*}

Since the first PSO algorithm proposed, the several PSO algorithms have been introduced with plethora of alterations.  Recently, the combination of quantum computing, mathematics and computer science have inspired the creation of optimization techniques. Initially,  Narayanan and Moore \cite{narayanan1996quantum} introduced quantum-inspired genetic algorithm (QGA) in 1995. Later, Sun et al. \cite{sun2004particle} applied the quantum laws of mechanics to PSO and proposed quantum-inspired particle swarm optimization (QPSO). It is the commencement of quantum-behaved optimization algorithms, which has subsequently made a significant impact on the academic and research communities alike.

Recently, Yuanyuan and Xiyu \cite{yuanyuan2018quantum} proposed a quantum evolutionary algorithm to discover communities in complex social networks. Its applicability is tested on five real social networks and results are compared with classical algorithms.  It has been proved that PSO lacks convergence on local optima i.e. it is tough for PSO to come out of the local optimum once it confines into optimal local region. QPSO with mutation operator (QPSO-MO) is proposed to enhance the diversity to escape from local optimum in search \cite{liu2005quantum}. Protopopescu  and Barhen \cite{protopopescu2002solving} solved set of global optimization problems efficiently using quantum algorithms. In future, the proposed algorithm can be integrated with matrix product state based quantum classifier  for supervised learning \cite{40, 44}.

In this paper, we have combined QPSO with Cauchy mutation operator to add long jump ability for global search and natural selection mechanism for elimination of particles. The results shown that it has great tendency to overcome the problem of trapping into local search space. Therefore, the proposed hybrid QPSO strengthened the local and global search ability and outperformed the other variants of QPSO and PSO due to fast convergence feature.

The illustration of particles movement in PSO and QPSO algorithm is shown in Fig 1. The big circle at center denotes the particle with the global position and other circles are particles. The particles located away from global position are lagged particles.  The blue color arrows signify the directions of other particles and the big red arrows point towards the side in which it goes with high probability. During iterations, if the lagged particle is unable to find better position as compared to present global position in PSO, then their impact is null on the other particles. But, in QPSO, the lagged particles move with higher probability in the direction of gbest position. Thus, the contribution of lagged particles is more to the solution in QPSO in comparison with PSO algorithm.

The organization of rest of this paper is as follows: Sect. 2 is devoted to prior work. In Sect. 3, the quantum particle swarm optimization is described. In
Sect. 4, the proposed hybrid QPSO algorithm with Cauchy distribution and natural selection mechanism  is presented. The experimental results are
plotted for a set of benchmark problems and compared with several QPSO variants in Sect. 5. The correctness and time complexity are analyzed in Section 6. QPSO-CD is applied to three constrained engineering design problems in Sect. 7.
Finally, Sect. 8 is the conclusion.

\section{Prior Work}
Since the quantum-behaved particle swarm optimization was proposed, various revised variants have been emerged. Initially, Sun et al. \cite{sun2004particle} applied the concept of quantum computing to PSO and developed a quantum Delta potential well model for classical PSO \cite{sun2004global}. It has been shown that the convergence and performance of QPSO are superior as compared to classical PSO. The selection and control of parameters can improve its performance, which is posed as an open problem. Sun et al. \cite{sun2007using} tested the performance of QPSO on constrained and unconstrained problems. It has been claimed that QPSO is a promising optimization algorithm, which performs better than classical PSO algorithms. In 2011, Sun et al. \cite{sun2011quantum} proposed QPSO with Gaussian distribution (GAQPSO) with the local attenuator point and compared its results with several PSO and QPSO counterparts. It has been proved that GAQPSO is efficient and stable with superior features in quality and robustness of solutions.

Further, Coelho \cite{dos2010gaussian} applied GQPSO to constrained engineering problems and shown that the simulation results of GQPSO are much closer to the perfect solution with small standard deviation. Li et al. \cite{li2012improved} presented a cooperative QPSO using Monte Carlo method (CQPSO), where particles cooperate with each other to enhance the performance of original algorithm. It is implemented on several representative functions and performed better than the other QPSO algorithms in context of computational cost and quality of solutions. Peng et al. introduced \cite{peng2013quantum} QPSO with Levy probability distribution and claimed that there are very less chances to  be stuck in local optimum.

Researchers have applied PSO and QPSO to real-life problems and achieved optimal solutions as compared to existing algorithms. Ali  et al. \cite{90} performed energy-efficient clustering in mobile ad-hoc networks (MANET) with PSO. The similar approach can be followed to analyse and execute mobility over MANET with QPSO-CD \cite{88}. Zhisheng \cite{zhisheng2010quantum} used QPSO in economic load dispatch for power system and proved superior to other existing PSO optimization algorithms. Sun et al. \cite{sun2006qpso} applied QPSO for QoS multicast routing. Firstly, the QoS multicast routing is converted into constrained integer problems and then effectively solved by QPSO with loop deletion task. Further, the performance is investigated on random network topologies. It has been proved that QPSO is more powerful than PSO and genetic algorithm.
Geis and Middendorf proposed PSO with Helix structure for finding ribonucleic acid (RNA) secondary structures with same structure and low energy \cite{geis2011particle}. The QPSO-CD algorithm can be used with two-way quantum finite automata to model the RNA secondary structures \cite{bhatia2018modeling}. 
 Bagheri et al. \cite{bagheri2014financial} applied the QPSO for tuning the parameters of adaptive network-based fuzzy inference system (ANFIS) for forecasting the financial prices of future market. Davoodi et al. \cite{davoodi2014hybrid} introduced a hybrid improved QPSO with Neldar Mead simplex method (IQPSO-NM), where NM method is used for tuning purpose of solutions. Further, the proposed 
algorithm is applied to solve load flow problems of power system and acquired the convergence accurately with efficient search ability. Omkar \cite{omkar2009quantum} proposed QPSO for multi-objective design problems and results are compared with PSO. Recently, Fatemeh et al. \cite{fatemeh2019shuffled} proposed QPSO with shuffled complex evolution (SP-QPSO) and its performance is demonstrated using five engineering design problems. Prithi and Sumathi \cite{prithi2020ld2fa} integrated the concept of classical PSO with deterministic finite automata for transmission of data and intrusion detection. The proposed algorithm QPSO-CD can be used with quantum computational models for wireless communication \cite{10, 20, 30}.

\section{Quantum Particle Swarm Optimization}
Before we explain our hybrid QPSO-CD algorithm mutated with Cauchy operator and natural selection method, it is useful to define the notion of quantum PSO. We assume that the reader is familiar with the concept of classical PSO; otherwise, reader can refer to particle swarm optimization algorithm \cite{kennedy2010particle, shi2001particle}. The specific principle of quantum PSO is given as:

In QPSO, the state of a particle can be represented using wave function $\psi(x, t)$. The probability density function $|\psi(x, t)|^2$ is used to determine the probability of particle occurring in  position \textit{x} at any time \textit{t} \cite{sun2004particle, sun2006qpso}. The position of particles is  updated according to equations:

\begin{equation}
x_{i, j}(t+1)=p_{i, j} (t) \pm \alpha . |mbest_{i, j}(t)-x_{i, j}(t)|. ln(1/u) 
\end{equation}
\begin{equation}
p_{i, j}(t)=  (\phi.P_{i, j}(t)+(1-\phi).G_j(t)),  (1 \leq i \leq  N, 1 \leq j \leq M)
\end{equation}
where each particle must converge to its local attractor $p=(p_1, p_2,..., p_D)$, where \textit{D} is the dimension, \textit{N} and \textit{M} are the number of particles and iterations respectively, $P_{i, j}$ and $G_{j}$ denote the previous and optimal position vector of each particle respectively, $\phi=c_1.r_1/(c_1r_1+c_2r_2)$, where $c_1$; $c_2$ are the acceleration coefficients, $r_1$; $r_2$ and \textit{u} are normally distributed random numbers in (0, 1), $\alpha$ is contraction-expansion coefficient and mbest defines the mean of best positions of particles as:
\begin{equation}
mbest_{i, j}(t)=\dfrac{1}{N} \sum_{i=1, j=1}^{N, M}P_{i, j}(t)=\bigg(\dfrac{1}{N}\sum_{i=1}^{N}P_{i, 1}(t), \dfrac{1}{N}\sum_{i=1}^{N}P_{i, 2}(t), . . ., \dfrac{1}{N}\sum_{i=1}^{N}P_{i, D}(t)  \bigg) 
\end{equation}
In Eq. (1), $\alpha$ denotes contraction-expansion coefficient, which is setup manually to control the speed of convergence. It can be decreased linearly or fixed. In PSO, $\alpha < 1.782$ to ensure convergence performance of the particle. In QPSO-CD, the value of $\alpha$ is determined  by $\alpha$=1-(1.0-0.5)\textit{k}/\textit{M}, i.e. decreases linearly from 1.0 to 0.5 to attain good performance, where \textit{k} is present iteration and \textit{M} is maximum number of iterations.

\section{Hybrid Particle Swarm Optimization}
The hybrid quantum-behaved PSO algorithm with Cauchy distribution and natural selection strategy (QPSO-CD) is described as follows:
	
The QPSO-CD algorithm begins with the standard QPSO using equations (1), (2) and (3). The position and velocity of particles cannot be determined exactly due to varying dynamic behavior.  So, it can only be learned with the probability density function. Each particle can be mutated with Gaussian or Cauchy distribution. We mutated QPSO with Cauchy operator due to its ability to make larger perturbation.  Therefore, there is a higher probability with Cauchy as compared to Gaussian distribution to come out of the local optima region. The QPSO algorithm is mutated with Cauchy distribution to increase its diversity, where mbest or global best position is mutated with fixed mutation probability ({\textit{Pr}). The probability density function $(f(x))$ of the standard Cauchy distribution is given as:
	
	\begin{equation}
	f(x)=\dfrac{1}{\pi(1+x^2)} ~ ~ ~ ~ -\infty < x < \infty,
	\end{equation}

It should be noted that mutation operation is executed on each vector by adding Cauchy distribution random value (\textit{D}(.)) independently such that
\begin{equation}
x^{'}=x+\phi D(.)
\end{equation} 	
where $x^{'}$ is new location after mutated with random value to \textit{x}. At last, the position of particle is selected and the particles of swarm are sorted on the basis of their fitness values after each iteration. Further, substitute the group of particles having worst fitness values with the best ones and optimal solution is determined.  The main objective of using natural mechanism is to refine the capability and accuracy of QPSO algorithm.
 	
 \begin{center}
\begin{minipage}{.7\linewidth}
\begin{algorithm}[H]
		\caption{QPSO-CD algorithm}
		\label{pseudoPSO}
		\begin{algorithmic}[1]
			\State The swarm is initialized with random numbers distributed uniformly: random $x_i$. 
			\State Do
			\State $\alpha$ decreases linearly from 1.0 to 0.5 
			
			\For {\textit{k}=1 to \textit{M} do}
			\State $\alpha=1-(1.0-0.5).k/M$
			\If {$Pr<rand(0,1)$}
			\State Calculate the mbest of the swarm using Eq (3)
			\EndIf
			\For {\textit{i}=1: to \textit{N} do}
			\If {Fitness$(p_i) <$ Fitness$(x_i)$} $p_i=x_i$; 
			\State $G=argmin(Fitness(p_i)$;
			\EndIf
			\For{\textit{j}:1 to \textit{D} do}
			\State $r_{1}=rand(0, 1); r_{2}=rand(0,1)$;
			\State $\phi=c_{1}.r_{1}/(c_{1}r_{1}+c_{2}r_{2})$;
			\State $p_{i, j}=\phi.P_{i, j}+(1-\phi). G_j$;
			\If {$rand(0,1) < 0.5$}
			\State $x_{i,j}=p_{i, j}+\alpha. abs(mbest_{i, j}-x_{i, j}). log (1/u)$
			\Else
			\State $x_{i,j}=p_{i, j}-\alpha. abs(mbest_{i, j}-x_{i, j}). log (1/u)$
			\EndIf
			\EndFor
			\State $Fx(i)$=Fitness$(x(i,:))$;
			\EndFor
			
			\State $[SF, Sx]=sort(Fx)$; 
			\State $Z=round((N-1)/S)$; \Comment{\textit{S} is selection parameter}
			\State $x(Sx((N-Z+1):N))=x(Sx(1:Z))$;  \Comment{Sort the particles from best to worst position}
			\EndFor
			\State Until termination criterion is met
			
		\end{algorithmic}
	\end{algorithm}
	 \end{minipage} 
	 \end{center}

The natural selection method is used to enhance the convergence characteristics of proposed QPSO-CD algorithm, where the fitter solutions are used for the next iteration. The procedure of selection method for \textit{N} particles is as follows:
\begin{equation}
F(X(t))=\{F(x_1(t)), F(x_2(t)), ..., F(x_N(t))\}
\end{equation}
where $X(t)$ is position vector of particles at time \textit{t} and $F(X(t))$ is the fitness function of swarm.  Next step is to sort the particles according to their fitness values from best one to worst position such that 
\begin{equation}
\begin{split}
F(X^{'}(t))=\{F(x_1^{'}(t)),   F(x_2{'}(t)) , ..., &F(x_N{'}(t))\} \\
X^{'}(t)=\{x_1^{'}(t), x_2{'}(t), ..., x_N{'}(t)\}
\end{split}
\end{equation}

In Algorithm 1, $SF$ and $Sx$ are the sorting functions of fitness and position respectively. On the basis of natural selection parameters and fitness values, the positions of swarm particles are updated for the next iteration,
\begin{equation}
\begin{split}
X^{'}(t)=\{x_1^{''}(t), x_2{''}(t), ...,  x_S{''}(t)\}, & \\
X^{''}_{k}(t)=\{x_1^{'}(t), x_2{'}(t), ..., x_Z{'}(t)\}
\end{split}
\end{equation}
where $(1 \leq k \leq S)$, \textit{S} denotes the selection parameter, \textit{Z} signifies the number of best positions selected according to fitness values such that $S=N/Z$ and $X^{''}(t)$ is updated position vector of particles. The selection parameter \textit{S} is generally set as 2 to replace the half of worst positions with the half of best positions of particles. It improves the precision of the direction of particles, protect the global searching capability and speed up the convergence.
	
\section{Experimental results}

The performance of proposed QPSO-CD algorithm is investigated on representative benchmark functions, given in Table 1. Further, the results are compared with classical PSO (PSO), standard QPSO, QPSO with delta potential (QDPSO) and QPSO with mutation operator (QPSO-MO). The details of numerical  benchmark functions are given in Table 1.
\begin{table}[!ht]
	\centering
	\caption{Details of benchmark functions}
	\begin{tabular}{|c|c|}
		\hline
		Test function  & Initial range  \\
		\hline
		\multicolumn{2}{|c|}{\textbf{Sphere function}}\\
		\hline
		$f_1(x)=\sum\limits_{i=1}^{n}x_{i}^2$ & (-100, 100)  \\
		\hline
		\multicolumn{2}{|c|}{\textbf{Rosenbrock function}} \\
		\hline
		$f_2(x)=\sum\limits_{i=1}^{n}100(x_{i+1}-x_{i}^2)^2+(x_{i}-1)^2$ & (-5.12, 5.12)  \\
		\hline
		\multicolumn{2}{|c|}{\textbf{Greiwank function}}\\
		\hline
		$f_3(x)=\dfrac{1}{4000} \sum\limits_{i=1}^{n}x_{i}^2-\prod_{i=1}^{n}cos(\dfrac{x_i}{\sqrt{i+1}})+1$& (-600, 600) \\
		\hline
		\multicolumn{2}{|c|}{\textbf{Rastrigrin function}} \\
		\hline
		$f_4(x)=\sum\limits_{i=1}^{n}(x_{i}^2-10cos(2\pi x_{i})+10)$ & (-5.12, 5.12) \\
		\hline
	\end{tabular}
\end{table}

The performance of QPSO has been widely tested for various test functions. Initially, we have considered four representative benchmark functions to determine the reliability of QPSO-CD algorithm. For all the experiments, the size of population is 20, 40 and 80 and dimension sizes are 10, 20 and 30. The parameters for QPSO-CD algorithm are as follows: the value of $\alpha$ decreases from 1.0 to 0.5 linearly, the natural selection parameter \textit{S}=2 is taken, $c_1, c_2$ correlation coefficients are set equal to 2.

The mean best fitness values of PSO, QPSO, QDPSO, QPSO-MO and QPSO-CD are recorded for 1000, 1500 and 2000 runs of each function. Fig. 2 to Fig. 5 depict the performance of functions $f_1$ to $f_4$ with respect to mean best fitness against the number of iterations. In Table 2, \textit{P} denotes the population, dimension is represented by \textit{D} and \textit{G} stands for generation. The numerical results of QPSO-CD shown optimal solution with fast convergence speed and high accuracy. The results shown that QPSO-CD performs better on Rosenbrock function than its counterparts in some cases. When the size of population is 20 and dimension is 30, the results of proposed algorithm are not better than QPSO-MO, but QPSO-CD performs better than PSO, QPSO and QDPSO. The performance of QPSO-CD is significantly better than its variants on Greiwank and Rastrigrin functions. It has outperformed other algorithms and obtained optimal solution (near zero) for Greiwank function. In most of the cases, QPSO-CD is more efficient and outperformed the other algorithms.  

\begin{figure} 
\centering
\begin{minipage}{.5\textwidth}
  \centering
  \includegraphics[width=1\linewidth]{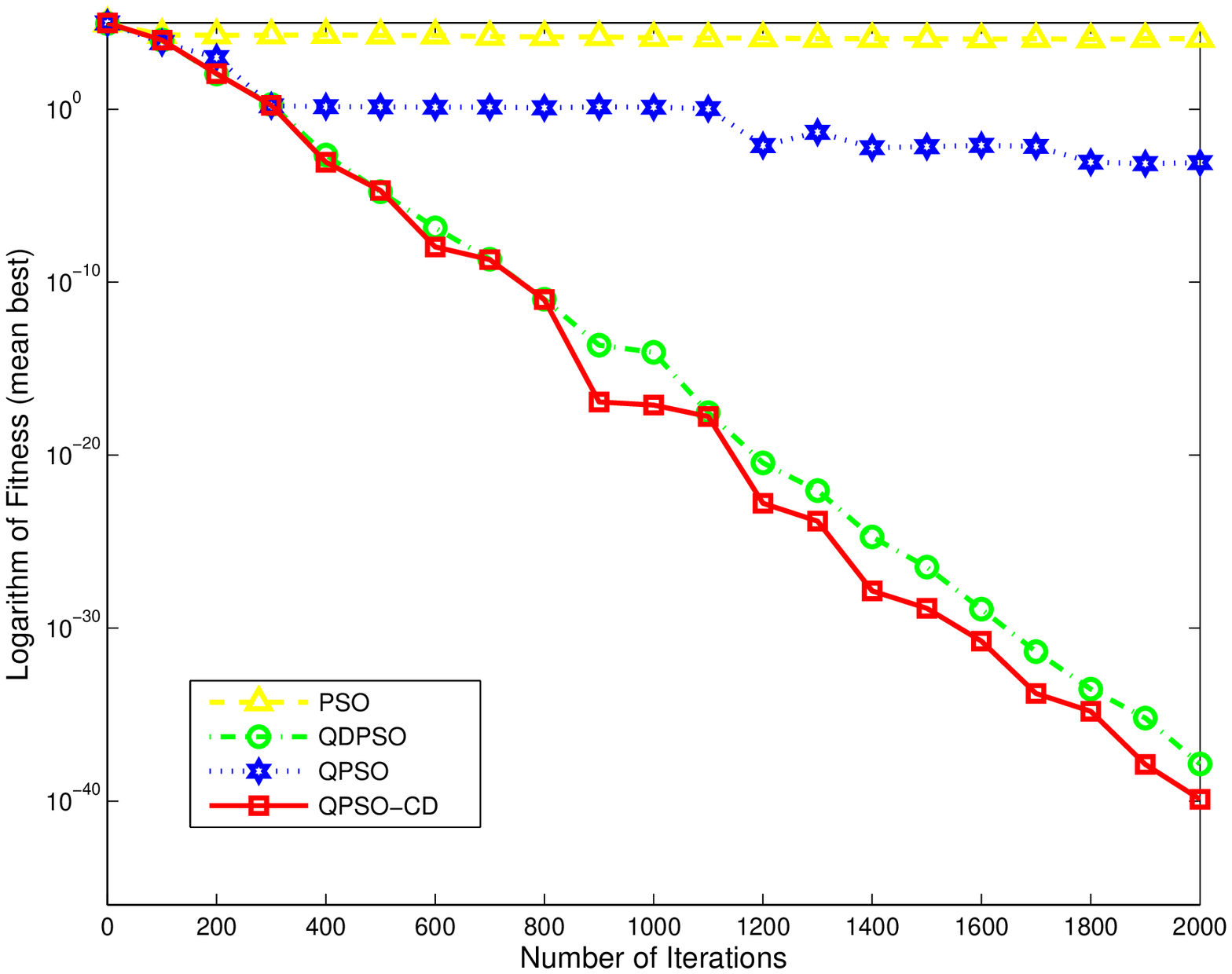}
  \captionof{figure}{Effectiveness of QPSO-CD for sphere function $f_1$}
  \label{fig:test1}
\end{minipage}%
\begin{minipage}{.5\textwidth}
  \centering
  \includegraphics[width=1\linewidth]{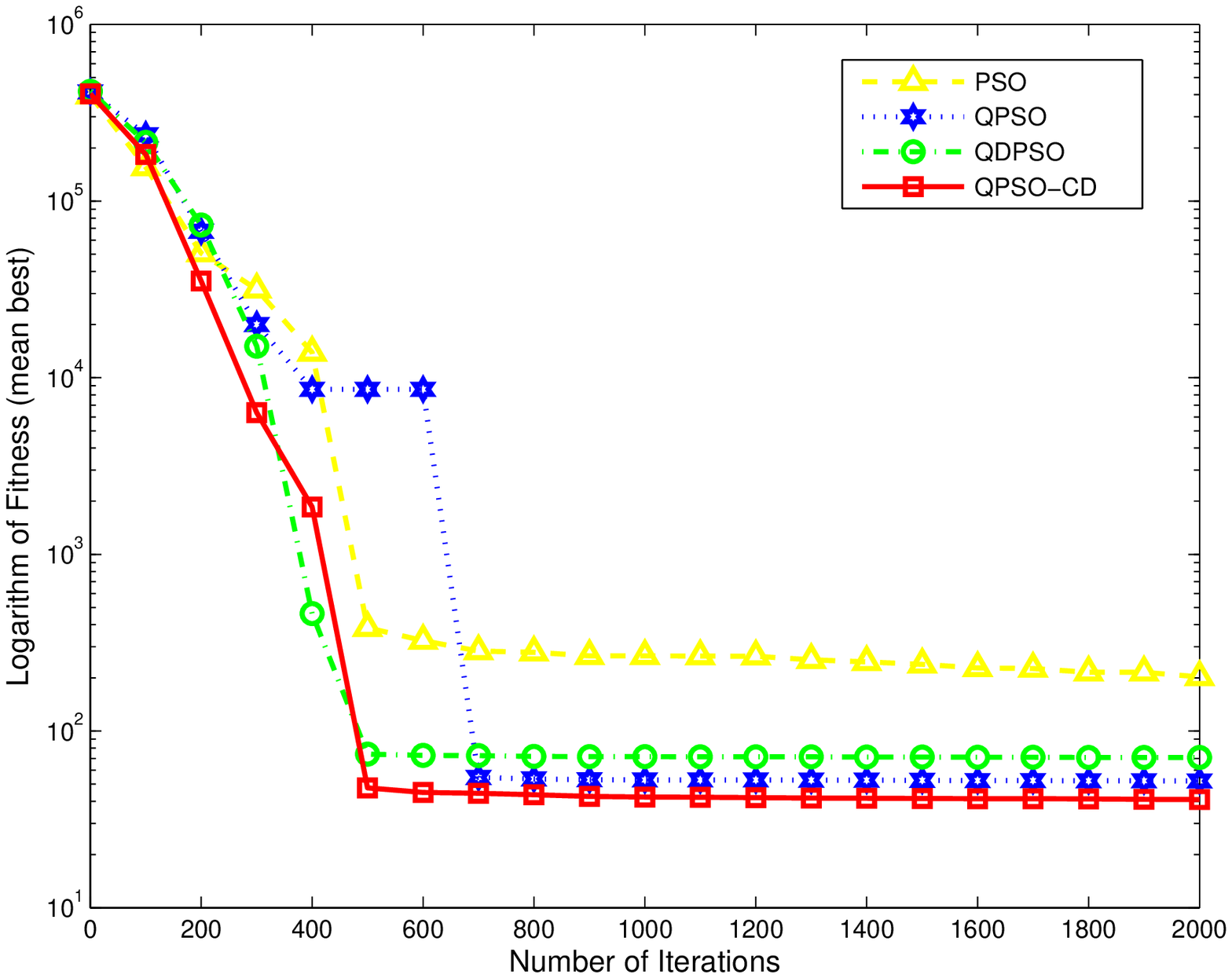}
  \captionof{figure}{Effectiveness of QPSO-CD for Rosenbrock function $f_2$}
  \label{fig:test2}
\end{minipage}
\end{figure}

\begin{figure} [!ht]
\centering
\begin{minipage}{.5\textwidth}
  \centering
  \includegraphics[width=1\linewidth]{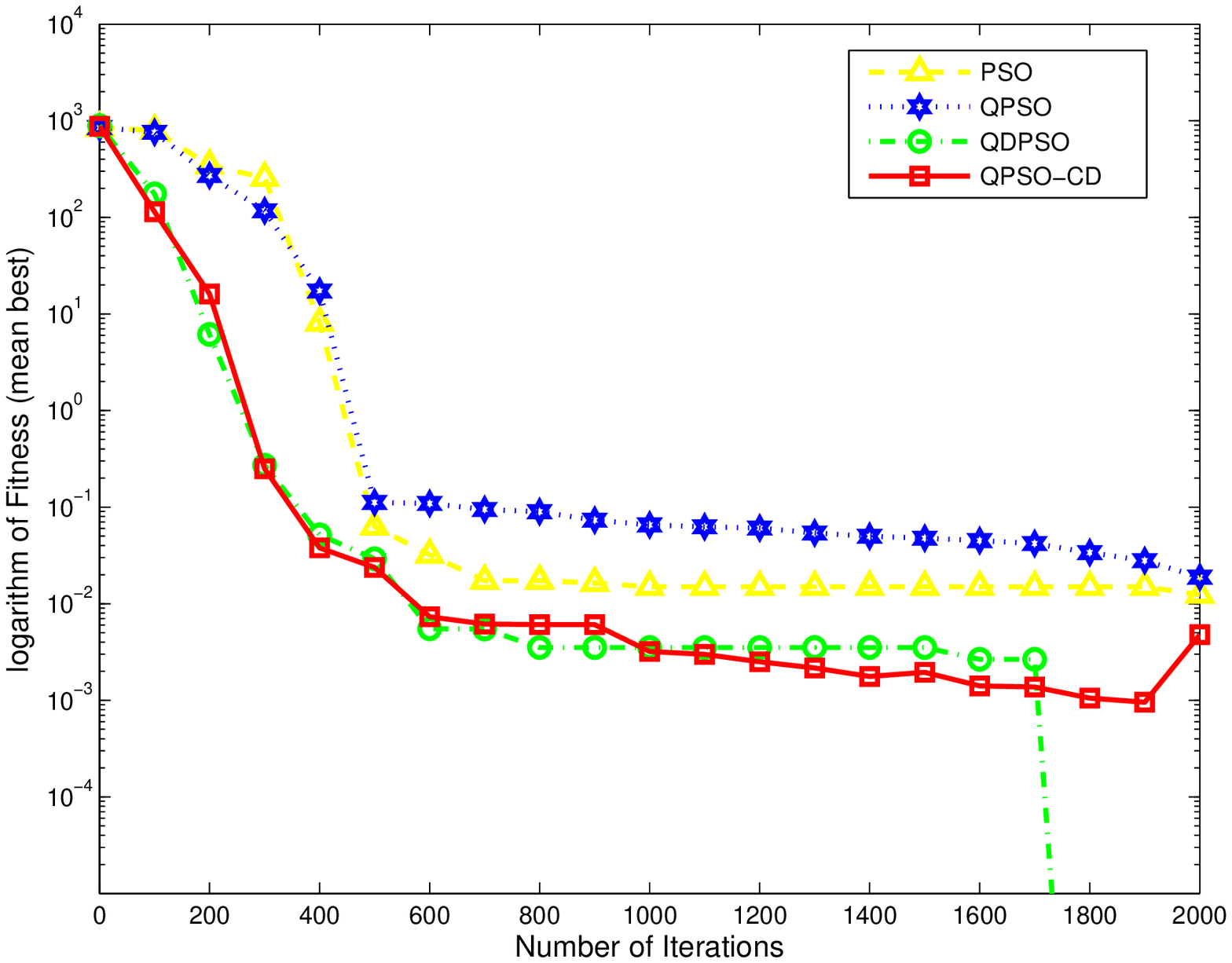}
  \captionof{figure}{Effectiveness of QPSO-CD for Greiwank function $f_3$}
  \label{fig:test1}
\end{minipage}%
\begin{minipage}{.5\textwidth}
  \centering
  \includegraphics[width=1\linewidth]{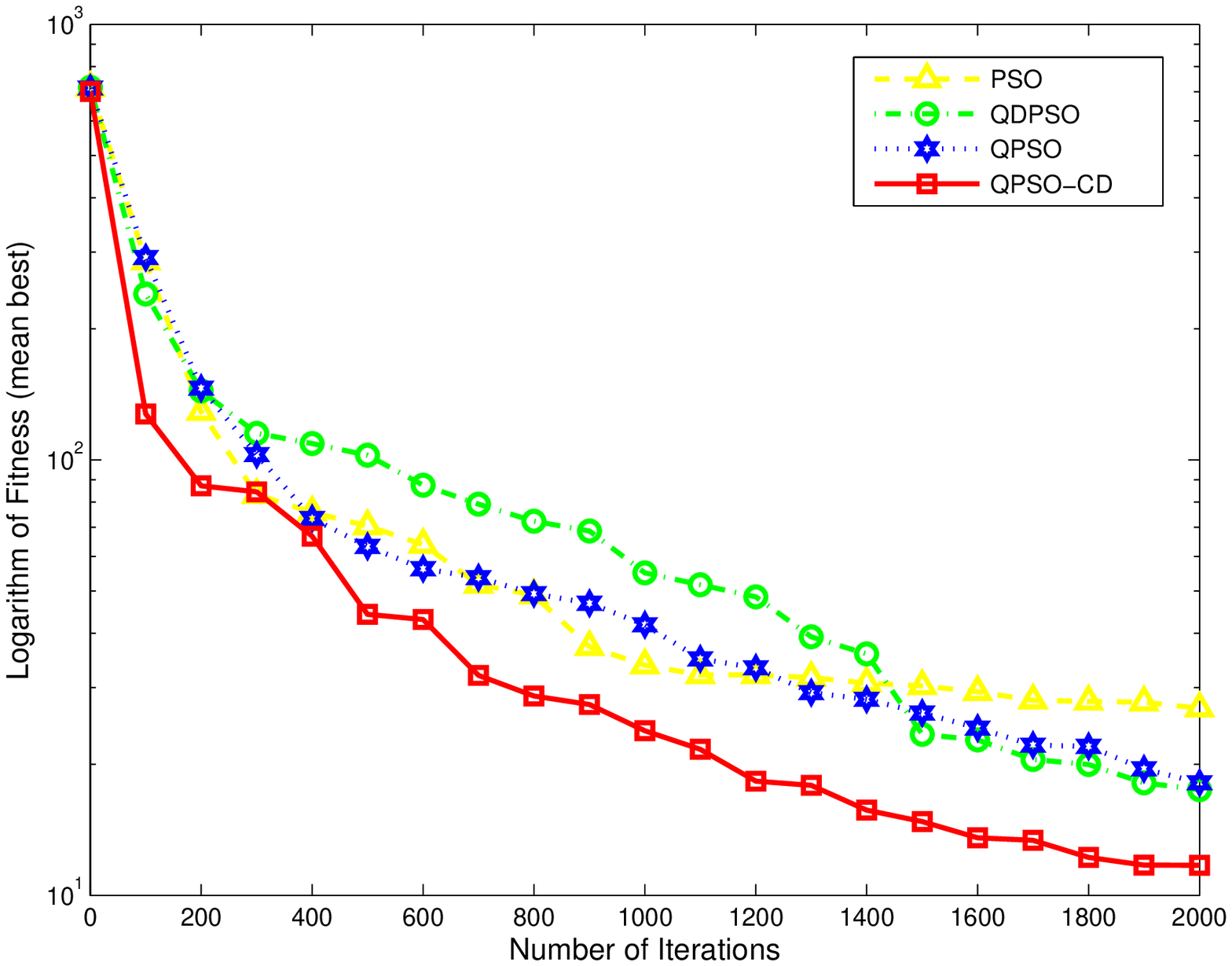}
  \captionof{figure}{Effectiveness of QPSO-CD for Rastrigrin function $f_4$}
  \label{fig:test2}
\end{minipage}
\end{figure}
\begin{table}[!ht]
	\centering
	\caption{Comparison results of Sphere and Rosenbrock functions}
	\begin{tabular}{c|c|c|c|c|c|c|c|c|c|c|c|c}
		\hline
		\multicolumn{3}{c|}{}& \multicolumn{5}{c|}{\textbf{Sphere function}} &  \multicolumn{5}{c}{\textbf{Rosenbrock function}}\\
		\hline
		P & D & G & PSO & QPSO & QDPSO & QPSO-MO & QPSO-CD & PSO & QPSO & QDPSO & QPSO-MO & QPSO-CD \\
		\hline
		\multirow{3}{*}{20}& 10 & 1000 &0.0  & 4.01e-40  & 1.513e-49  &1.508e-48  & 1.738e-50 & 95.10 & 58.41 & 14.22 & 22.18 & 34.67  \\
		& 20 & 1500 & 0.0 & 2.58e-21  & 1.339e-30   & 1.296e-31 & 1.032e-30 & 204.38 & 110.5 & 175.31 & 68.40 & 54.76 \\
		& 30 & 2000 & 0.0 & 2.08e-13  & 1.953e-21  & 1.918e-21  & 1.808e-21 & 314.46 & 148.5 & 242.37 & 113.30 & 122.5 \\
		\hline
		\multirow{3}{*}{40}& 10 & 1000 & 0.0 & 2.73e-67 & 1.087e-73  & 1.146e-51 & 1.154e-72 & 70.28 & 10.42 & 15.86 & 7.985 &  8.843    \\
		& 20 & 1500 & 0.0  & 4.84e-28 & 1.397e-42  & 1.417e-42  & 1.237e-41 & 178.98 & 48.45 & 112.46 & 52.93 & 41.77 \\
		& 30 & 2000 & 0.0 & 2.02e-25  & 2.850e-30  &  2.471e-28 & 1.946e-23  & 288.58 & 58.32 & 76.42 & 64.19& 58.04 \\
		\hline
		\multirow{3}{*}{80}& 10 & 1000 &  0.0 & 7.66e-95 & 5.553e-90  & 4.872e-71 &  6.437e-72 & 36.29 & 8.853 & 36.34 & 5.715 & 7.419  \\
		& 20 & 1500 & 0.0 & 1.62e-60 &  1.654e-54  & 1.677e-58  & 1.609e-62 & 84.78 & 34.88 & 23.54 &24.45 & 21.78 \\
		& 30 & 2000 & 0.0 & 2.05e-44 & 1.042e-40  & 1.131e-42  & 1.128e-41 & 202.58 & 52.17 & 70.81&45.22 & 40.97 \\
		\hline
	\end{tabular}
\end{table}

\begin{table}[!ht]
	\centering
	\caption{Comparison results of Greiwank and Rastrigrin functions}
	\begin{tabular}{c|c|c|c|c|c|c|c|c|c|c|c|c}
		\hline
		\multicolumn{3}{c|}{}& \multicolumn{5}{c|}{\textbf{Greiwank function}} &  \multicolumn{5}{c}{\textbf{Rastrigrin function}}\\
		\hline
		P & D & G & PSO & QPSO & QDPSO & QPSO-MO & QPSO-CD & PSO & QPSO & QDPSO & QPSO-MO & QPSO-CD \\
		\hline
		\multirow{3}{*}{20}& 10 & 1000 & 0.089 & 0.078 & 0.1003 &0.0732 & 0.072 & 5.526 & 5.349 & 4.969 & 4.478 & 4.051   \\
		& 20 & 1500 & 0.0300 & 0.2001 &  0.0086 & 0.0189 & 0.0078 & 23.17 & 21.28 & 17.08 & 15.63 & 13.22\\
		& 30 & 2000 & 0.0181 & 0.0122 & 0.0544& 0.0103 & 0.0026 & 46.29 & 32.57 & 48.61 & 27.80 & 31.48 \\
		\hline
		\multirow{3}{*}{40}& 10 & 1000 &0.0826 & 0.055 & 0.048 & 0.0520 &  0.041 & 3.865 & 3.673 & 2.032 & 3.383 &  2.100 \\
		& 20 & 1500 & 0.0272 & 0.0149 & 0.0004 & 0.0247& 0.0106 & 15.68 & 14.37 & 10.94 & 11.01 & 10.77\\
		& 30 & 2000 & 0.0125 & 0.0117 & 0.0009& 0.0105& 0.0102 & 37.13 & 23.01 & 21.37& 21.01& 21.19\\
		\hline
		\multirow{3}{*}{80}& 10 & 1000 & 0.0723 & 0.0341 &0.0 & 0.0542& 0.0702 & 2.562 & 2.234 & 0.923 & 2.183&  1.943  \\
		& 20 & 1500 & 0.0274  &  0.0189 & 0.0 & 0.0194& 0.0161 & 12.35 & 9.66 & 6.955 & 8.075& 7.021 \\
		& 30 & 2000 & 0.0123 & 0.0118 & 0.0 & 0.0082 & 0.0031 & 26.89 & 17.48 & 18.13 &14.99 & 11.73 \\

		\hline
	\end{tabular}
\end{table}

\section{Correctness and Time Complexity Analysis of a QPSO-CD Algorithm}
In this Section, the correctness and time complexity of a proposed algorithm QPSO-CD is analyzed and compared with the classical PSO algorithm.
\begin{theorem}
	The sequence of random variables $\{S_n, n \geq 0\}$ generated by QPSO with Cauchy distribution converges to zero in probability as \textit{\textit{n}} approaches infinity. 
		\end{theorem}
	\begin{proof}
		Recall, the probability density function of standard Cauchy distribution and its convergence probability \cite{rudolph1997local} are given as
	\begin{equation}
	\begin{gathered}
	f(s)=\dfrac{1}{\pi (1+s^2)} ~ \text{for} -\infty < s < \infty, \\
P(x \leq S_n \leq y) = \dfrac{1}{\pi} \int_{y}^{x} \dfrac{ds}{(1+s^2)}, ~ \forall x \leq y	
\end{gathered}
	\end{equation}	
		Consider a random variable $Q_n$ interpreted as
		$$ Q_n = \alpha S_n, ~ \alpha =\dfrac{1}{n^\lambda} $$
		where $\lambda$ denotes a fixed positive constant. Correspondingly, the probability density function can be calculated as
		\begin{align*}
		P(Q_n \leq q) & = P(\alpha S_n \leq q) \\
		& = P\left(S_n \leq \dfrac{q}{\alpha}\right)\\
		& = \int_{-\infty}^{\dfrac{q}{\alpha}} f(s) ds. \dfrac{dP(q_n \leq q)}{dq}\\
		& = \dfrac{1}{\alpha} f\left(\dfrac{q}{\alpha}\right)	
		\end{align*}
	i.e. the probability density function of random variable $Q_n$.
	\begin{align*}
	P(|Q_n|>\xi) & = P(|S_n|>\xi n^\lambda)\\
	& = P(S_n > \xi n^\lambda)+ P(S_n > - \xi n ^\lambda)\\
	& = P(\xi n^\lambda < S_n < \infty)+ P(- \infty < S_n < - \xi n^\lambda)
   \end{align*}
		Using Eq. (9), the probability density function of random variable $Q_n$ becomes
		\begin{align*}
			P(|Q_n|>\xi) & = \dfrac{1}{\pi}\int_{\xi n^\lambda}^{\infty}\dfrac{ds}{\pi (1+s^2)} + \dfrac{1}{\pi}\int_{\infty}^{-\xi n^\lambda}\dfrac{ds}{\pi (1+s^2)} \\
			& = \left[1+\dfrac{1}{\pi} \int_{\xi n^\lambda}^{-\xi n^\lambda}\dfrac{ds}{\pi (1+s^2)}\right]\\
			& = 1-\dfrac{1}{\pi} \int_{-\xi n^\lambda}^{\xi n^\lambda}\dfrac{ds}{\pi (1+s^2)}= 0 ~ \text{as} ~ n \rightarrow \infty
		\end{align*}
		
		This completes the proof of the theorem.
	\end{proof}
\begin{defn}
	Let $\{S_n\}$ a random sequence of variables. It is converges to some random variable \textit{s} with probability 1, if for every $\xi > 0$ and $\lambda >0$, there exists $n_1(\xi, \lambda)$ such that $P(|S_n-s|< \xi) > 1-\lambda, \forall n > n_1$ or 
	\begin{equation}
	P\left(\lim_{n \rightarrow \infty} |S_n-s|< \xi\right)=1
	\end{equation}
\end{defn}

The efficiency of the QPSO-CD algorithm is evaluated by number of steps needed to reach the optimal region $R(\xi)$. The method is to evaluate the distribution of number of steps needed to hit $R(\xi)$ by comparing the expected value and moments of distribution. The total number of stages to reach the optimal region is determined as $W(\xi)=inf\{n \mid f_n  \in R(\xi)\}$. The variance $V(W(\xi))$ and expectation value $E(W(\xi))$ are determined as 

\begin{equation}
E(W(\xi))= \sum_{n=0}^{\infty}nx_n,
\end{equation}

\begin{equation}
\begin{split}
V(W(\xi)) & =E(W^2(\xi))-\{E(W(\xi))\}^2\\
& = \sum_{n=0}^{\infty}n^2x_n-\left(\sum_{n=0}^{\infty}nx_n\right)^2
\end{split}
\end{equation}

In fact, the $E(W(\xi))$ depends upon the convergence of $\sum_{n=0}^{\infty}nx_n$. It is needed that $\sum_{j=0}^{\infty}x_n=1$, so that QPSO-CD can converge globally. The number of objective function  evaluations are used to measure time. The main benefit of this approach is that it shows relationship  between processor and measure time as the complexity of objective function increases. We used Sphere function  $f(x)=x^{T}.x$ with a linear constraint $g(x)=\sum^{n}_{j=0}x_j \geq 0$ to compute the time complexity. It has minimum value at 0. The value of optimal region is set as $R(\xi)=R(10^{-4}).$ To determine the time complexity, the algorithms PSO and QPSO-CD are executed 40 times on $f(x)$ with initial scope [-10, 10]$^{N}$, where \textit{N} denotes the dimension. We determine the mean number of objective function evaluations ($W(\xi)$), the variance ($V(W(\xi))$), the standard deviation (SD) ($\sigma_{W(\xi)}$), the standard error (SE) ($\sigma_{W(\xi)}/ \sqrt{40}$)  and ratio of mean and dimension ($W(\xi)/N$). The contraction coefficient $\alpha= 0.75$ is used for QPSO-CD and constriction coefficient $\chi=0.73$ for PSO with acceleration factors $c_1=c_2$=2.25.

\begin{table}[!ht]
	\centering
	\caption{Results of the time complexity for QPSO-CD algorithm}
	\begin{tabular}{c|c|c|c|c|c}
		\hline
		Dimension (N) & Mean & Variance & SD & SE & Mean/N  \\
		\hline
	2	 &302.38  & 4164.3 & 64.532   &10.203 & 151.19 \\
	3	& 452.18   & 4541.9 & 67.394   &10.655 & 150.72\\
	 4 & 621.29   & 5208.2 & 72.168 &  11.410  & 155.32\\
	 5 &755.88 & 6675.0 & 81.701   & 12.918 & 151.17\\
	 6 & 879.13 & 8523.7 & 92.324    & 14.597 & 146.52\\
	 7 &  1022.06 & 9575.4  & 97.854      & 15.472 & 146.00\\
	 	8 & 1158.52  & 10269.7  & 101.341   & 16.023 & 144.81\\
	 	9 & 1308.17  & 12053.4 &   109.788  & 17.359 & 145.35 \\
	 	10 & 1459.3 & 12648.3  &  112.465   & 17.782 & 145.93\\
		\hline
	\end{tabular}
\end{table}

\begin{table}[!ht]
	\centering
	\caption{Results of the time complexity for PSO algorithm}
	\begin{tabular}{c|c|c|c|c|c}
		\hline
		Dimension (N) & Mean & Variance & SD & SE & Mean/N  \\
		\hline
		2	 &691.4  & 17297.5 & 131.52   & 20.795 &  345.7\\
		3	& 979.1   & 22281.5&  149.27  & 23.601& 326.3\\
		4 & 1167.2   & 24282.9 & 155.72 &  24.638 & 291.8\\
		5 & 1328.7 & 21853.7 & 147.83  &  23.373& 265.7\\
		6 & 1489.9 &  32008.7&   178.91 &  28.288& 248.3 \\
		7 &  1744.3 & 502297  &  224.12    & 35.436 & 249.1 \\
		8 & 1978.5  &  41233.3 &  203.06  & 31.106 & 247.3 \\
		9 & 2259.1  & 36217.8 & 190.31  & 30.090&  251.0\\
		10 & 2604.2 &  43559.8 &  208.71   &  32.999& 260.4\\
		\hline
	\end{tabular}
\end{table}
\begin{figure*}[!ht]
	\centering
	\includegraphics[scale=0.6]{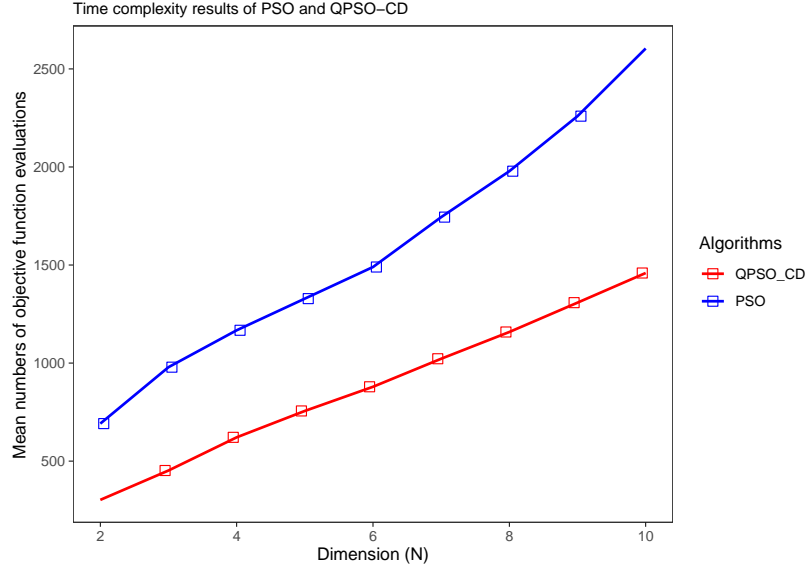}
	\caption{Time complexity results for PSO and QPSO-CD}
\end{figure*}
\begin{figure*}[!ht]
	\centering
	\includegraphics[scale=0.6]{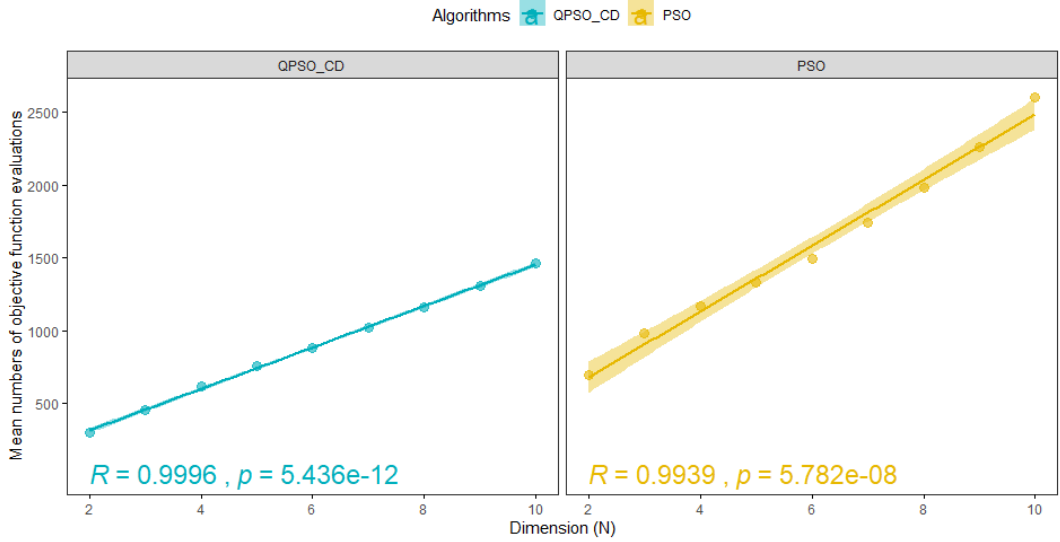}
	\caption{Comparison of Correlation coefficient of PSP and QPSO-CD}
\end{figure*}

Table 1 and 2 show the statistical results of time complexity test for QPSO-CD and PSO algorithm, respectively. Fig 6 indicates that the time complexity of proposed algorithm increases non-linearly as the dimension increases. However, the time complexity of PSO algorithm increases adequately linearly. Thus, the time complexity of QPSO-CD is lower than PSO algorithm. In Fig 7, QPSO-CD shows a strong correlation between $W(\xi)$ and \textit{N}, i.e. the correlation coefficient is 0.9996. For PSO, the linear correlation coefficient is 0.9939, which is not so phenomenal as that in case of QPSO-CD.  The relationship between  mean and dimension clearly shows that the value of correlation coefficient is fairly stable for QPSO-CD as compared to PSO algorithm.
\section{QPSO-CD for Constraint Engineering Design Problems}
 There exists several approaches for handling constrained optimization problems. The basic principle is to convert the constrained optimization problem to unconstrained by combining objective function and penalty function approach. Further, minimize the newly formed objective function with any unconstrained algorithm. Generally, the constrained optimization problem can be described as in Eq. (13).

The objective is to minimize the objective function \textit{f(x)} subjected to equality $(h_j(x))$ and inequality $(g_i(x))$ constrained functions, where $p(i)$ is the upper bound and $q(i)$ denotes the search space lower bound. The strict inequalities of form $g_i(x) \geq 0$ can be converted into $-g_i(x)\leq 0$ and  $h_i(x)$ equality constraints can be converted into inequality constraints $h_i(x) \geq 0$ and $h_i(x) \leq 0$. Sun et al. \cite{sun2007using} adopted non-stationary penalty function to address non-linear programming problems using QPSO. Coelho \cite{dos2010gaussian} used penalty function with some positive constant i.e. set to 5000. We adopted the same approach and replace the constant with dynamically allocated penalty value.
\begin{equation}
\begin{split}
& \min\limits_{x}=f(x) \\
\text{subject to}  \\
& g_i(x) \leq 0, i=0, 1,... n-1\\
& h_j(x)=0, j=1,2,...r \\
& p(i)\leq x_i \leq q(i), 1 \leq i \leq m\\
& x=\{x_1,x_2, x_3,..., x_m\} 
\end{split}
\end{equation}

Usually, the procedure is to find the solution for design variables that lie in search space upper and lower bound constraints such that $x_i \in [p(i), q(i)]$. If solution violates any of the constraint, then the following rules are applied
\begin{equation}
\begin{split}
& x_i=x_i + \{p(x_i)-q(x_i)\}.~ rand[0,1] \\
& x_i=x_i - \{p(x_i)-q(x_i)\}.~ rand[0,1]
\end{split}
\end{equation}

where rand[0, 1] is randomly distributed function to select value between 0 and 1. Finally, the unconstrained optimization problem is solved using dynamically modified penalty values according to inequality constraints $g_i(x)$. Thus, the objective function is evaluated as 

\begin{equation}
F(x)=\left\{
\begin{array}{ll}
f(x)~ ~ ~ ~ ~ ~ ~ ~ ~ ~ ~ ~ ~ ~ ~ ~ ~ ~ ~ ~ ~ ~ ~\text{if}~ g_i(x)\leq 0\\
f(x)+y(t). \sum\limits_{i=1}^{n} g_i(x)~ ~  \text{if}~ g_i(x)>0\\
\end{array} \right \}
\end{equation}

where $f(x)$ is the main objective function of optimization problem in Eq. (13), \textit{t} is the iteration number and $y(t)$ represents the dynamically allocated penalty value.

In this Section, QPSO-CD is tested for three-bar truss, tension/compression spring  and pressure vessel design problems consisting
different members and constraints. The performance of QPSO-CD is compared and analyzed with the results of PSO, QPSO, and SP-QPSO algorithms as reported in the literature.
\subsection{Three-bar truss design problem}
Three-bar truss is a constraint design optimization problem, which has been widely used to test several methods. It consists cross-section areas of three bars $x_1$ (and $x_3$) and $x_2$ as design variables. The aim of this problem is to  minimize the weight of truss subject to maximize the stress on these bars. The structure should be symmetric and subjected to two constant loadings $P_1=P_2=P$ as shown in Fig 8. The mathematical formulation of two design bars ($x_1$, $x_2$) and three restrictive mathematical functions are described as:

\begin{equation}
\begin{split}
\text{Min}  ~  f(x)& = (2\sqrt{2}x_1+x_2). l \\
\text{Subject to:} & \\
g_1(x)&=\dfrac{\sqrt{2}x_1+x_2}{\sqrt{2}x_1^{2}+2x_1x_2}P-\sigma \leq 0,\\
g_2(x)&= \dfrac{x_2}{\sqrt{2}x_1^{2}+2x_1x_2}P-\sigma \leq 0 \\
\end{split}
\end{equation}
$$
g_3(x)= \dfrac{1}{x_1+\sqrt{2}x_2}P-\sigma \leq 0, ~ \text{where}$$
$$ 0 \leq x_1, x_2  \leq 1, l=100cm, P=\sigma=2 KN/cm^{2}$$
\begin{figure}[ht]
	\centering
	\includegraphics[scale=0.75]{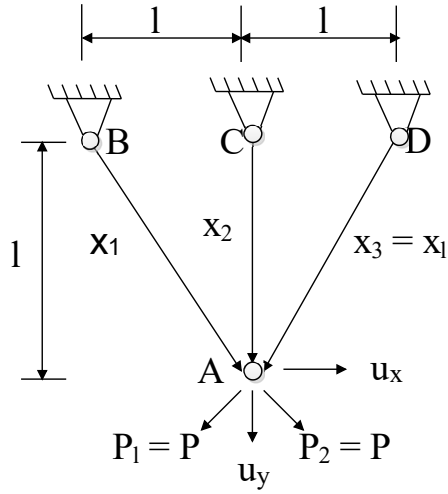}
	\caption{Structure of Three-bar truss}
\end{figure}

\begin{table}[!ht]
	\centering
	\caption{Comparison of optimal results for three-bar truss problem}
	\begin{tabular}{c|c|c|c|c}
		\hline
		Variables & PSO & QPSO & SP-QPSO & QPSO-CD  \\
		\hline
		$x_1$ &0.78911058  &0.788649  &  0.788796  & 0.788658\\
		$x_2$&  0.40702683 & 0.408322 &    0.407898& 0.40828488\\
		$g_1(x)$ & -6.6720e-06 & 1.6313e-07 & 6.4748e-06 &  9.00037e-06 \\
		$g_2(x)$ & -1.4655 & -1.4640 &   -1.4644 & -1.4640 \\
		$g_3(x)$ & -0.5345 &-0.5359  &    -0.5354 & -0.5359\\	
		$f(x)$ & 263.89686 &  263.89584 &     263.89500 & 263.89465\\	
		\hline
	\end{tabular}
\end{table}
The results are obtained by QPSO-CD are compared with its counterparts in Table 6. For three-bar truss problem, QPSO-CD is superior to optimal solutions  previously obtained in literature. The difference of best solution obtained by QPSO-CD among other algorithms is shown in Fig 9.

\begin{figure}[H]
	\centering
	\includegraphics[scale=0.55]{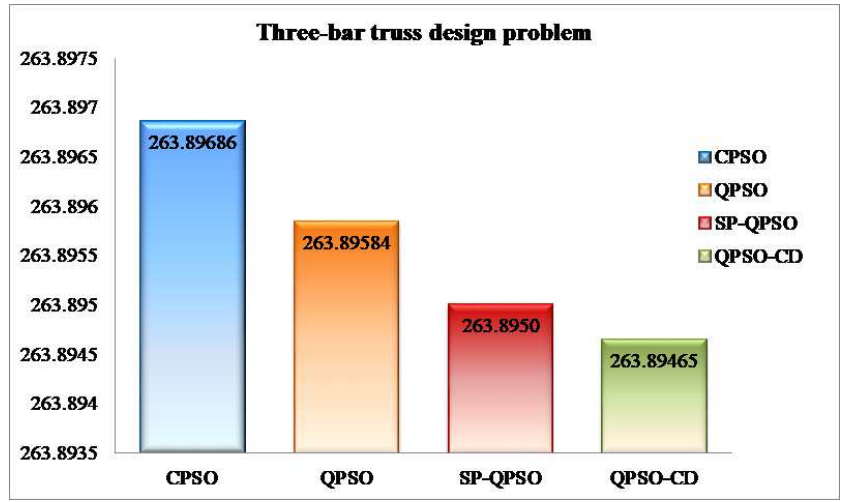}
	\caption{Optimal results of PSO, QPSO, SP-QPSO and QPSO-CD algorithms for three bar truss problem}
\end{figure}

\subsection{Tension/Compression spring design problem}
The main aim is to lessen the volume \textit{V} of a spring
subjected to tension load constantly as shown in Fig 10. Using the symmetry of structure, there are practically three design variables ($x_1, x_2, x_3$), where $x_1$ is the wire diameter, the coil diameter is represented by $x_2$ and $x_3$ denotes the total number of active coils.    The mathematical formulation for this problem is described as:

\begin{equation}
\begin{split}
\text{Min}  ~  f(x)&= (x_3+2)x_2x_1^2, \\
\text{Subject to:} & \\
g_1(x)&=\dfrac{1-x_2^{3}x_3}{71785x_16{4}} \leq 0,\\
g_2(x)&= \dfrac{4x_2^{2}-x_1x_2}{12566(x_2x_1^{3}-x_1^{4})}+\dfrac{1}{5108x_1^{2}}-1 \leq 0,\\
g_3(x)&=  \dfrac{1-140.45x_1}{x_2^2x_3}\leq 0,\\
g_4(x)&=  \dfrac{x_2+x_1}{1.5}-1\leq 0,\\
\text{where} & \\
0.05 \leq x_1 & \leq 2,  0.25  \leq x_2, \leq 1.3, 2 \leq x_3  \leq 15, \\
\end{split}
\end{equation}

	\begin{figure}[H]
	\centering
	\includegraphics[scale=0.38]{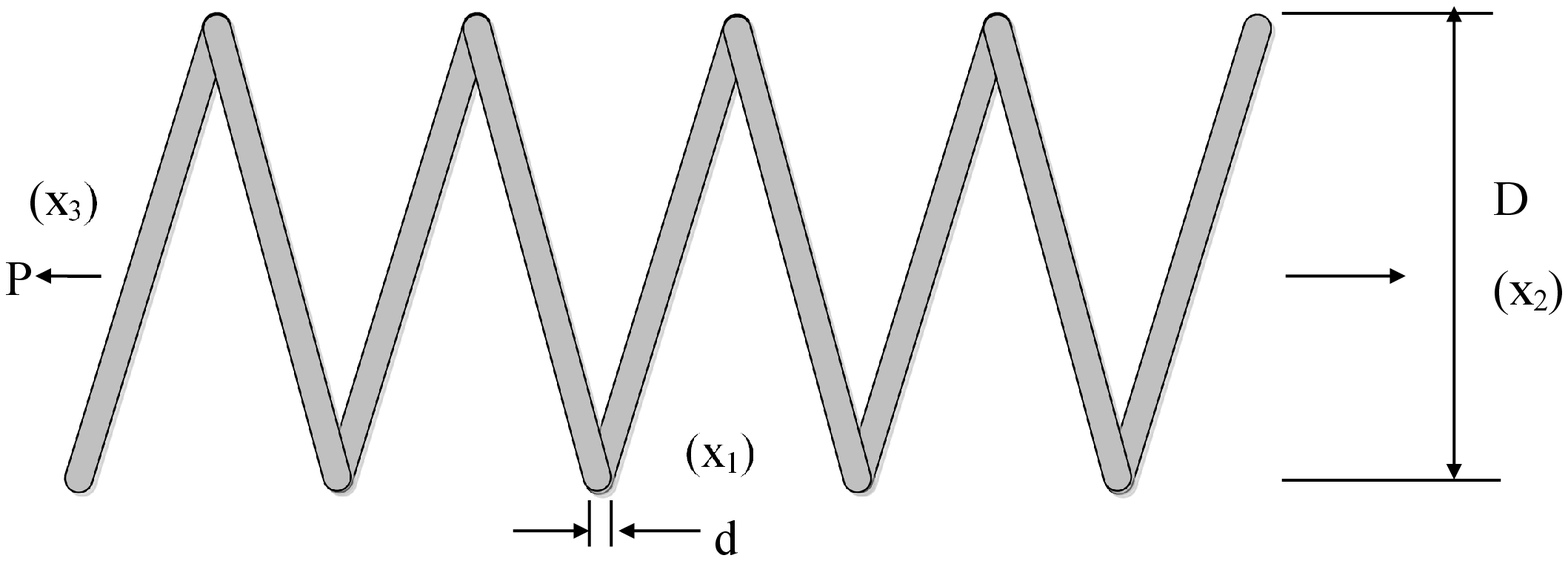}
	\caption{Structure of tension/compression spring}
\end{figure}
\begin{table}[!ht]
	\centering
	\caption{Comparison of optimal results for tension spring design problem}
	\begin{tabular}{c|c|c|c|c}
		\hline
		Variables & PSO & QPSO & SP-QPSO & QPSO-CD  \\
		\hline
		$x_1$ &  0.0516 & 0.0524  & 0.05 & 0.0513   \\
		$x_2$&  0.3542 & 0.2505  &  0.25  & 0.2502  \\
		$x_3$ & 11.7942 & 2 &  2 & 2\\
		$g_1(x)$ & -2.3006e-02  & 0.93145095  & 0.93034756 & 4.11004e-06  \\
		$g_2(x)$ &  -5.6059e-03  &-0.17471558  & -0.16568318 & -0.17352479  \\
		$g_3(x)$ & -3.9057 & -50.67  & -48.180  &  -49.561 \\	
		$g_4(x)$ & -0.7294 & -0.79986567 & -0.80   & -0.799 \\	
		$f(x)$ & 0.01305 &0.00275  &  0.00250  &  0.00263\\	
		\hline
	\end{tabular}
\end{table}

It has been observed that QPSO algorithm with Cauchy distribution and natural selection strategy is robust and obtains optimal solutions than 
PSO and QPSO, shown in Table 7. The difference between best solutions found by QPSO-CD  ($f(x)=0.00263$) and other algorithms for tension spring design problem are reported in Fig 11.

\begin{figure}[H]
	\centering
	\includegraphics[scale=0.55]{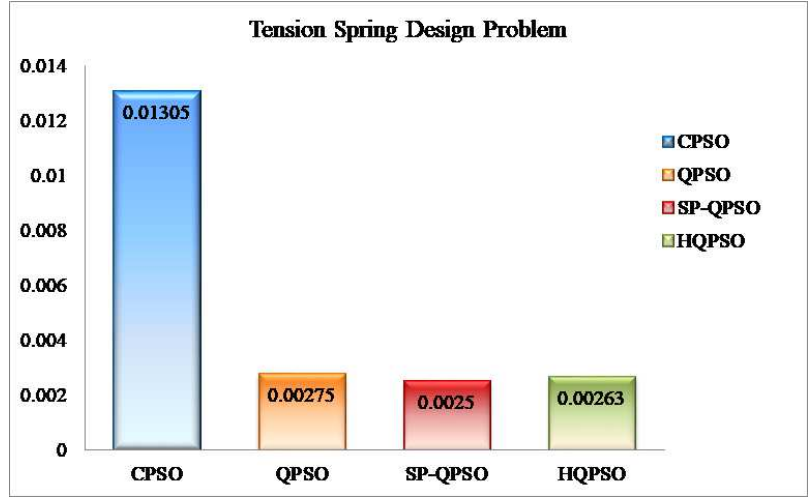}
	\caption{Results of PSO, QPSO, SP-QPSO and QPSO-CD methods for tension spring design problem}
\end{figure}

\subsection{Pressure vessel design problem}

Initially,  Kannan and Kramer \cite{kannan1994augmented} studied the pressure vessel design problem with the main aim to reduce the total fabricating cost. Pressure
vessels can be of any shape. For engineering purposes, a cylindrical design capped by hemispherical heads at both ends is widely used \cite{sandgren1990nonlinear}. Fig 12 describes the structure of pressure vessel design problem. It consists four design variables ($x_1, x_2, x_3, x_4$), where $x_1$ denotes the shell thickness $(T_s)$, $x_2$ is used for head thickness ($T_h$), $x_3$ denotes the inner radius (\textit{R}) and  $x_4$ represents the length of vessel (\textit{L}). The objective function and constraint equations are described as:
\begin{figure}[H]
	\centering
	\includegraphics[scale=0.62]{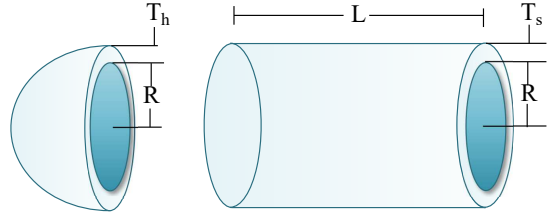}
	\caption{Design of pressure vessel}
\end{figure}

\begin{equation}
\begin{split}
\text{Min}  ~  f(x)& =  0.6224x_1x_3x_4 + 1.7781x_2x_3^{2}+  3.166x_1^{2}x_4 \\ & +19.84x_1^{2}x_3,  \\
\text{Subject to:} & \\
g_1(x)&= -x_1+0.0193x_3 \leq 0,\\
g_2(x)&=  -x_2+0.00954x_3\leq 0,\\
g_3(x)&= -\pi x_3^{2}x_4-\dfrac{4}{3}\pi x_3^{3}+1296000 \leq 0,\\
g_4(x)&= x_4-240 \leq 0,\\
\text{where} & \\
1\times 0.0625 & \leq x_1, x_2  \leq 99 \times 0.0625,  10  \leq x_3 ,x_4 \leq  200 \\
\end{split}
\end{equation}

\begin{table}[!ht]
	\centering
	\caption{Comparison of optimal results for Pressure vessel design problem}
	\begin{tabular}{c|c|c|c|c}
		\hline
		Variables & PSO & QPSO & SP-QPSO & QPSO-CD  \\
		\hline
		$x_1$ & 0.8125 &0.7783  & 0.7782   & 0.7776\\
		$x_2$&  0.4375& 0.3849 & 0.3845    & 0.3848\\
		$x_3$ & 42.0984 & 40.3289 & 40.3206 & 40.3278\\
		$x_4$ & 176.6365  & 199.8899 & 199.9988 & 199.8865 \\
		$g_1(x)$ & -4.500e-15 & 4.777e-05  & -1.242e-05  & 7.2654e-04\\
		$g_2(x)$ & -0.035880 & -1.62294e-04  & 1.58523e-04   & -7.2787e-05 \\
		$g_3(x)$ & -1.164e-10  & -97.39720071  & -63.63686942    & -0.734359\\	
		$g_4(x)$ & -63.3634 & -40.1100 &  -40.0012  & -40.1135 \\	
		$f(x)$ & 6059.714
		&5886.189  & 5885.268   & 5886.137\\	
		\hline
	\end{tabular}
\end{table}

\begin{figure}[!ht]
	\centering
	\includegraphics[scale=0.55]{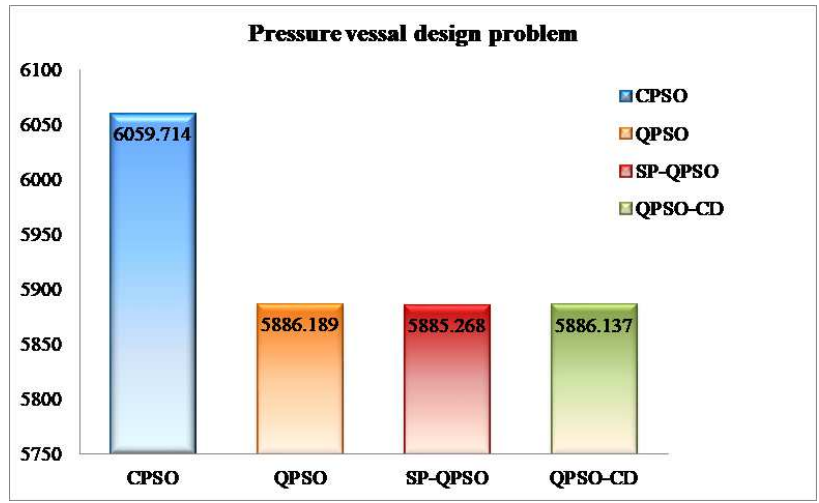}
	\caption{Optimal results of PSO, QPSO, SP-QPSO and QPSO-CD techniques for pressure vessel design problem}
\end{figure}

The optimal results of QPSO-CD is compared with the SP-QPSO, QPSO and PSO best results noted in the previous work, and are given in Table 8. The best solution obtained from QPSO-CD is  better than other algorithms as shown in Fig 13. 

\section{Conclusion}
In this paper, a new hybrid quantum particle swarm optimization algorithm is proposed with natural selection method and Cauchy distribution. The performance of the proposed algorithm is experimented on four benchmark functions and the optimal results are compared with existing algorithms. Further, the QPSO-CD is applied to solve engineering design problems. The efficiency of QPSO-CD 
is successfully presented with superiority  than preceding results
 for three engineering design problems: three-bar truss, tension/compression spring and pressure vessel. The efficiency of QPSO-CD algorithm is evaluated  by number of steps needed to reach the optimal region and proved that time complexity of proposed algorithm is lower in comparison to classical PSO.  In the context of convergence, the experimental outcomes shown that the QPSO-CD converge to get results closer to the superior solution.

\section*{Additional information}
\textbf{Competing interests:} The authors declare no competing interests.

\section*{Acknowledgement}
S.Z. acknowledges support in part from the National Natural Science
Foundation of China (Nos. 61602532),  the Natural Science Foundation of Guangdong Province of China (No. 2017A030313378), and the Science and Technology Program of Guangzhou City of China (No. 201707010194).

\end{document}